\documentclass[final]{colt2021} 
\usepackage[T1]{fontenc}

\usepackage{times}

\input{our_macros.sty}

\title{Impossibility of Partial Recovery in the Graph Alignment Problem}

\coltauthor{\Name{Luca Ganassali} \Email{luca.ganassali@inria.fr}\\
	\addr INRIA, DI/ENS, PSL Research University, Paris, France.
	\AND
	\Name{Marc Lelarge} \Email{marc.lelarge@ens.fr}\\
	\addr INRIA, DI/ENS, PSL Research University, Paris, France.
	\AND
	\Name{Laurent Massouli\'e} \Email{laurent.massoulie@inria.fr}\\
	\addr MSR-Inria Joint Centre, INRIA, DI/ENS, PSL Research University, Paris, France.
}

\begin{document}
\maketitle

\begin{abstract}%
Random graph alignment refers to recovering the underlying vertex correspondence between two random graphs with correlated edges. This can be viewed as an average-case and noisy version of the well-known graph isomorphism problem. For the correlated \ER model, we prove an impossibility result for partial recovery in the sparse regime, with constant average degree and correlation, as well as a general bound on the maximal reachable overlap. Our bound is tight in the noiseless case (the graph isomorphism problem) and we conjecture that it is still tight with noise. Our proof technique relies on a careful application of the probabilistic method to build automorphisms between tree components of a subcritical \ER graph.
\end{abstract}


\section{Introduction}
Graph alignment, also known as graph matching, aims at finding a bijective mapping between the  vertex sets of two graphs so that the number of  adjacency disagreements between the two graphs is minimized. It reduces to the graph isomorphism problem in the noiseless setting where the two graphs can be matched perfectly. The paradigm of graph alignment has found numerous applications across a variety of diverse fields, such as network privacy (\cite{Narayanan08}), computational biology (\cite{Singh08}), computer vision (\cite{CFVS04}), and natural language processing.

Given two graphs with adjacency matrices $A$ and $B$, the graph matching problem can be viewed as a special case of the quadratic assignment problem (QAP) (\cite{Pardalos94}): 
\begin{equation}\label{eq:QAP}
    \max_\Pi \langle A, \Pi B \Pi^T\rangle
\end{equation}
where $\Pi$ ranges over all $n\times n$ permutation matrices, and $\langle \cdot, \cdot \rangle$ denotes the matrix inner product. QAP is NP-hard in general. These hardness results are applicable in the worst case, where the observed graphs are designed by an adversary.  In many applications, the graphs can be modeled by random graphs; as such, our focus is not in the worst-case instances, but rather in recovering partially the underlying vertex permutation with high probability.

\paragraph{Correlated \ER model}
Driven by applications in social networks and biology, a recent line of work (\cite{Lyzinski14, Feizi16, Cullina2017, elel18,Cullina18, Ding18, Cullina18data, Fan2019Wigner, GLM19, Wu20, Ganassali20gaussian, fan2019ERC,Ganassali20a,wu2021settling}) initiated the statistical analysis of graph matching by assuming that matrices $A$ and $B$ are generated randomly. 
The simplest such model is the following \emph{correlated \ER model}: we are given two graphs $\cG$ and $\cG'$ with the same set of nodes $[n]$ and with respectively blue and red edges. The blue and red edges are obtained by sampling uniformly at random:
\begin{itemize}
    \item with probability $\lambda s/n$ to get two-colored edges;
    \item with probability $\lambda(1-s)/n$ to get a blue (monochromatic) edge;
    \item with probability $\lambda(1-s)/n$ to get a red (monochromatic) edge;
    \item with probability $1-\lambda(2-s)/n$ to get a non-edge,
\end{itemize}
where $\lambda>0$ and $s\in [0,1]$ are fixed parameters and $n$ is large.
Hence each $\cG$ and $\cG'$ is a sparse \ER model with edge probability $\lambda/n$. For large values of $n$, the fraction of edges of $\cG$ (resp. $\cG'$) that are shared with $\cG'$ (resp. $\cG$) is. of order $s$ (see Figure \ref{fig:img_GGp}).

\begin{figure}[H]
	\centering
	\includegraphics[scale=0.8]{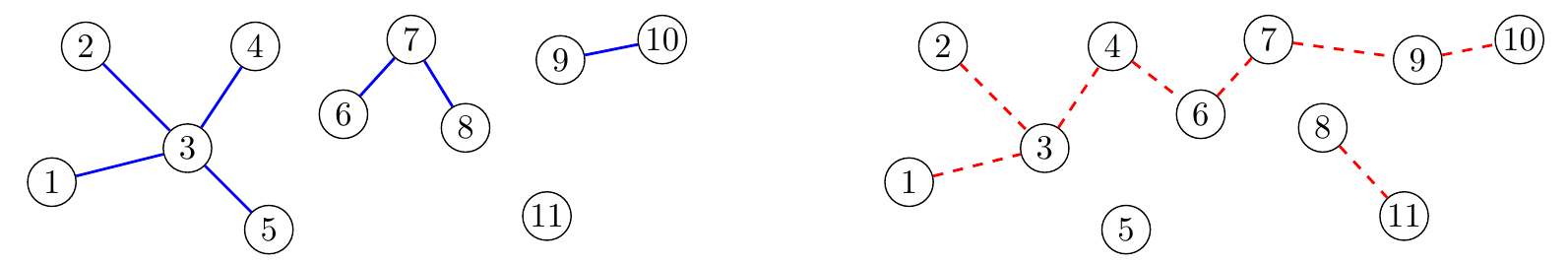}
	\caption{A realization of $(\cG,\cG')$ from the correlated \ER model, with $n=11$, $\lambda = 1.9$, and $s=0.7$. For the sake of readability, red edges are always dashed.}
	\label{fig:img_GGp}
\end{figure}

 We then relabel the vertices of the red graph $\cG'$ with a uniform independent permutation $\pi^* \in \cS_n$, and we observe $\cG$ and $\cH := \cG'^{\pi^*}$, see Figure \ref{fig:img_GH}.
Upon observing $\cG$ and $\cH$, the goal is to recover (or, reconstruct) partially the latent vertex correspondence $\pi^*$ with probability converging to $1$ as $n\to \infty$.

\begin{figure}[h]
	\centering
	\includegraphics[scale=0.8]{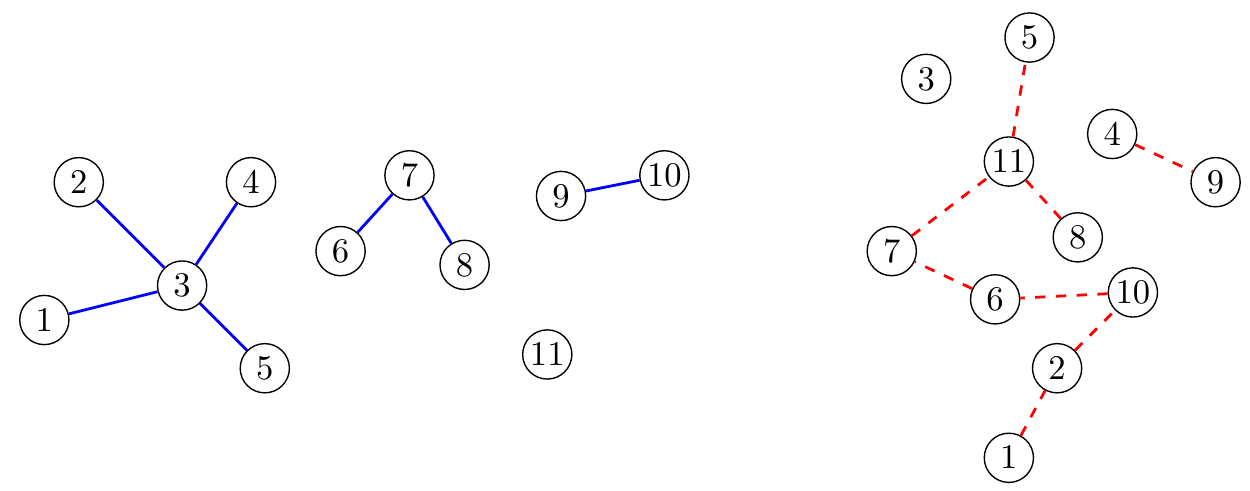}
	\caption{The pair $(\cG,\cH)$ corresponding to $(\cG,\cG')$ of Figure \ref{fig:img_GGp}, after relabeling $\cG'$ with the permutation $\pi^* = (6)(1 \; 5 \; 3 \; 11\; 9 \;2 \;8 \;4 \;7 \;10)$.}
	\label{fig:img_GH}
\end{figure}

\paragraph{Partial alignment in the sparse regime} We now define our notion of performance. First note that since we are in the sparse regime, even without any noise, i.e. with $s=1$, there is no way to be able to map the $\Theta(n)$ isolated vertices in $\cG$ and $\cH$ better than chance. Hence, we concentrate on the partial alignment problem where we ask for the best possible fraction of matched vertices between $\cG$ and $\cH$. More formally, an \emph{estimator} $\hat{\pi}$ (of $\pi^*$) is a $\cS_n$-valued measurable function of $(\cG, \cH)$. In order to match the two graphs correctly, we do not want to allow $\hat{\pi}$ to learn any information from the unique observation of $\cG$: indeed, since there is no canonical labeling of graphs, the estimator $\hat{\pi}$ must perform well even after any relabeling of the nodes of $\cG$. Hence, for any estimator $\hat{\pi}=\hat{\pi}(\cG,\cH)$ of $\pi^*$, we define its \emph{overlap} as follows
\begin{equation}\label{eq:overlap}
\ov(\hat{\pi}(\cG,\cH),\pi^*) := \frac{1}{n!} \sum_{\sigma \in \cS_n} \sum_{i=1}^{n} \mathbf{1}_{\hat{\pi}(\cG^{\sigma},\cH)(i)= \pi^* \circ  \sigma^{-1} (i)}.
\end{equation} With this definition, it is then easy to check that as wanted, for any $\sigma \in \cS_n$,
\begin{equation}\label{eq:invariance_overlap}
\ov(\hat{\pi}(\cG^{\sigma},\cH),\pi^*) = \ov(\hat{\pi}(\cG,\cH),\pi^* \circ \sigma).
\end{equation} 

\begin{remarque}
This natural definition \eqref{eq:overlap} is here to put aside trivial estimators such as $\hat{\pi}=\id$. Note that for the maximum a posteriori estimator $\hat{\pi}_{\MAP}$, which is the permutation solving the maximization problem \eqref{eq:QAP}, the first sum in \eqref{eq:overlap} can be simplified to one term. Indeed $\hat{\pi}_{\MAP}$ -- similarly to a lot of other 'natural' estimators -- verifies an \emph{equivariance} property, in the sense that for all $\sigma \in \cS_n$,
\begin{equation*}
	\hat{\pi}(\cG^{\sigma},\cH) = \hat{\pi}(\cG,\cH) \circ \sigma^{-1}.
\end{equation*}
\end{remarque}

Partial alignment thus consists in finding a estimator $\hat{\pi}$ of $\pi^*$ satisfying $\ov(\hat{\pi},\pi^*) > \alpha n$ with high probability, for some $\alpha>0$.  Let us start by stating a conjecture:
\begin{conj*}
\item[$(i)$] If $\lambda s \leq 1$, partial reconstruction is impossible, i.e. for any $\alpha >0$, for all estimator $\hat{\pi}$, $$\dP\left(\ov(\hat{\pi},\pi^*) > \alpha n  \right) \underset{n \to \infty}{\longrightarrow} 0.$$
\item[$(ii)$] If $\lambda s > 1$, partial reconstruction is possible (feasible), i.e. there exists $\alpha>0$ and an estimator $\hat{\pi}$ such that $$\dP\left(\ov(\hat{\pi},\pi^*) > \alpha n  \right) \underset{n \to \infty}{\longrightarrow} 1.$$
\end{conj*}

\paragraph{Results in the regime with constant mean degree and correlation}
In this paper, we work in the regime where $\lambda >0$ and $s \in [0,1]$ are fixed constants. Our results prove part $(i)$ of the conjecture, which had not been previously studied, and give an upper bound on the maximal reachable overlap in case $(ii)$. Let us mention straightaway the related results in our regime that are helpful for our conjecture: \cite{Ganassali20a} prove that partial recovery is possible (in polynomial time) in a region $\cR := \{(\lambda,s); \: \lambda\in [1,\lambda_0) \text{ and } s \in (s^*(\lambda),1] \}$ for some function $s^*(\lambda)<1$, so that interestingly the case $\lambda > \lambda_0$ is left open, nevertheless much in step with $(ii)$. Previous results from \cite{Hall20} showed that partial reconstruction was feasible for $\lambda s>C$, with an unspecified constant $C>20$. At the very time when this paper is being finished, new results from \cite{wu2021settling} are significantly improving these results, narrowing down the gap for $(ii)$. When translated with our notations, it is shown that partial alignment is possible (theoretically) if $\lambda s \geq 4+\eps$. These results are summed up in a diagram in Figure \ref{fig:diagram}. In particular, our bound is tight and our conjecture is almost solved for the case $s=1$, with a remaining gap $[\lambda_0,4]$ being still open. 

\begin{figure}[H]
	\centering
	\includegraphics[scale=0.92]{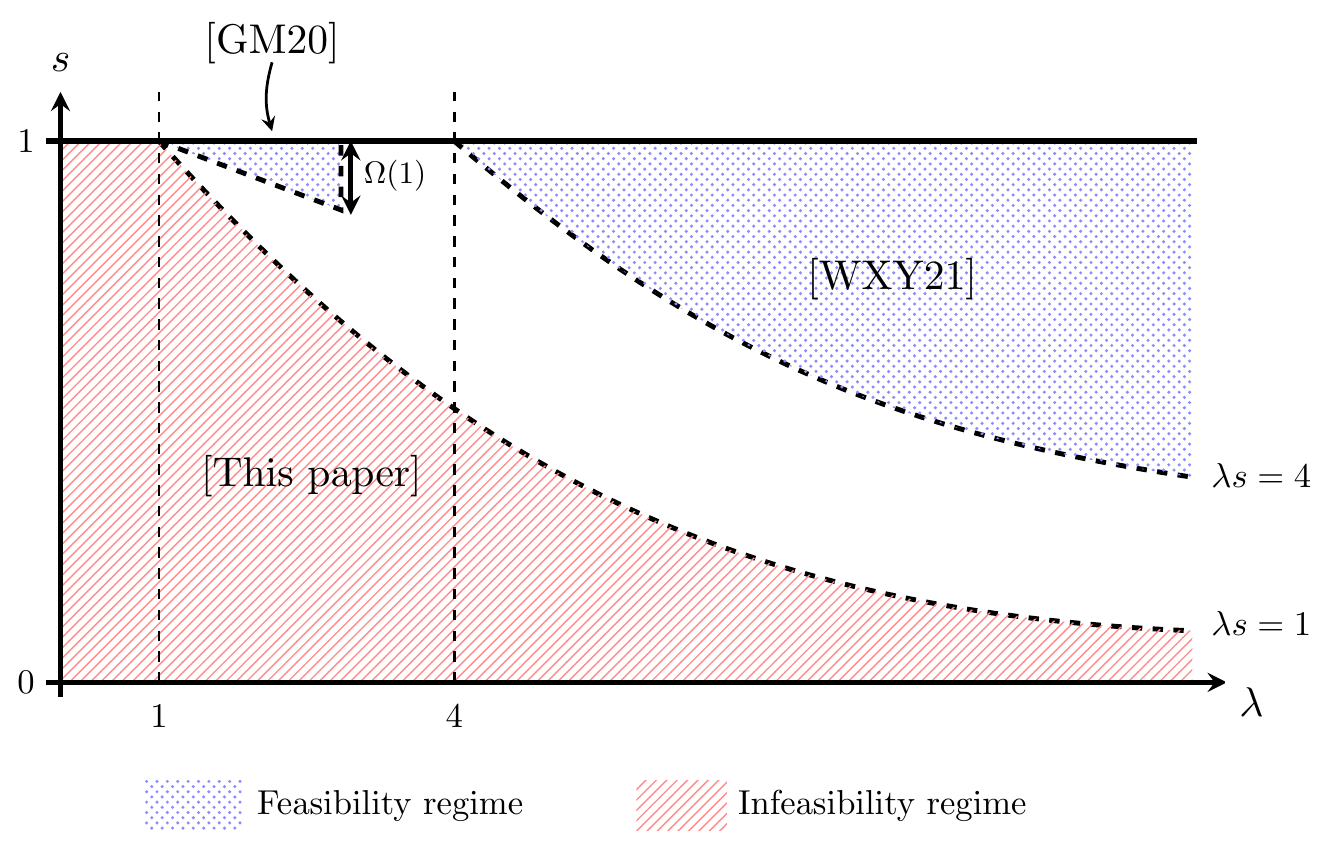}
	\caption{Diagram of the $(\lambda,s)$ regions where partial reconstruction is known to be impossible (resp. possible), in the sparse regime where $\lambda, s$ are fixed constants.}
	\label{fig:diagram}
\end{figure}

\paragraph{Main result} The main result of our paper is as follows:
\begin{theorem}\label{theorem:sparse_threshold} 
For $\lambda>0$ and $s\in[0,1]$, we have for any $\alpha>0$, for any estimator $\hat{\pi}$:
\begin{equation}
\dP\left(\ov(\hat{\pi},\pi^*) > (c(\lambda s)+ \alpha) n  \right) \underset{n \to \infty}{\longrightarrow} 0,
\end{equation}
where $c(\mu)$ is the greatest non-negative solution to the equation $e^{-\mu x}=1-x$.
\end{theorem}
Note that a well-known result (see e.g. \cite{Bollobas2001}) is that $c(\mu)$ is the typical fraction of nodes in the largest component of an \ER graph with average degree $\mu$, and that $c(\mu)=0$ if $\mu \leq 1$, and $c(\mu) \in (0,1)$ whenever $\mu >1$. Hence, Theorem \ref{theorem:sparse_threshold} implies that partial reconstruction is impossible for $\lambda s \leq 1$. Moreover, if $\lambda s > 1$, any estimator can reach an overlap of at most $c(\lambda s) n+o(n)$. Note that $c(\lambda s)$ is the typical fraction of nodes in the largest component of the intersection graph. 

\paragraph{Further related work} 
Graph matching has also been widely studied in the case where the mean degree and correlation are not fixed constants. The model is the same, with adapted notations: the probability for two-colored (resp. monochromatic, non-) edges are now $qs$ (resp. $q(1-s)$, $1-q(2-s)$), with $s \geq q$ and $q$ and $s$ that may depend on $n$. Note that our study focuses on the sparse setting where $q=\lambda/n$ and $s$ is constant, and that interesting results in other regimes are known for partial, almost exact, and exact recovery. We hereafter give them in detail, for completeness.

\begin{itemize}
\item Exact recovery (when $\ov(\hat{\pi},\pi^*)=n$) has been studied theoretically in recent work from \cite{wu2021settling} refining previous results of \cite{Cullina2017}, exhibiting a tight threshold at $$\frac{nq \left(\sqrt{s} - \sqrt{q}\right)^2}{\log n} = 1,$$

with the only condition that\footnote{$q/s$ is often referred to as the mean degree in the \emph{parent graph} of $\cG,\cG'$. Indeed, another common way of generating the two graphs under the correlated \ER model is to consider $\cF$ a parent \ER graph of $n$ nodes and mean degree $q/s$, and perform two independent sub-samplings of $\cF$, keeping each edge independently with probability $s$, forming $\cG$ and $\cG'$, two correlated \ER graphs of mean degree $q$.} $q/s$ is bounded away from 1.

\item For almost exact recovery (i.e. $\ov(\hat{\pi},\pi^*) \geq (1-\eps)n$ for all $\eps>0$), in a sparse regime where $q/s = n^{-\Omega(1)}$, it is known (\cite{Cullina18}) that almost exact recovery is possible if and only if $nqs \to \infty$. In a denser regime where $q/s = n^{-o(1)}$, \cite{wu2021settling} show that there is a tight threshold exhibiting an "all-or-nothing" phenomenon at
$$\frac{nq \left(s \log \frac{s}{q} - (s-q)\right)}{\log n} = 2,$$ above which almost exact recovery is possible and below which even partial recovery is impossible.

\item For partial recovery, the first investigation made by \cite{Hall20} -- though rather difficult to translate in our model -- showed that $nqs \to 0$ is an impossibility condition, whereas $nqs \geq C$ (with a large, non-explicit constant $C$), together with some additional sparsity constraints, ensures feasibility. As mentioned, \cite{wu2021settling} improve these results, showing that in the case $q/s = n^{-\Omega(1)}$, $nqs \geq 4+\eps$ suffices to ensure possibility. In addition, an impossibility condition of the form $n q s \leq 1-\eps$ is also established, but in a denser case, where $nq/s = \omega(\log^2 n)$. Note that this last impossibility result does not cover our regime, where both the mean degree $nq$ and the correlation parameter $s$ are of order $1$.

\end{itemize} 
For the impossibility part, \cite{wu2021settling} works with the mutual information $I(\pi^*; \cG, \cG')$, closely related to the minimum mean squared error. They are able to derive an upper bound on the expectation of $\ov(\hat{\pi},\pi^*)$, for any estimator, which happens to be $o(1)$ when the mean degree in the parent graph of $\cG$ and $\cG'$ is at least of order $\log^2 n$, but not when $\lambda, s$ are of order $1$. In our result, we do not work directly with the mutual information, but we are considering the posterior distribution of $\pi^*$: in simple words, we show that under the assumption $\lambda s<1$ the posterior distribution puts equal weights on permutations that overlap only on a vanishing fraction of points. This is done by building ad hoc permutations with the probabilistic method.

In this paper, we derive information-theoretic results: our proof is not constructive, i.e. not related to a particular algorithm. The search for efficient algorithms is a very active field of research: using spectral methods (\cite{Feizi16,fan2019ERC}), degree profiles (\cite{Ding18}), convex relaxation (\cite{dym2017ds++}), etc. Unfortunately, except from \cite{Ganassali20a}, these algorithms are not known to give a positive fraction of overlap in the regime $\lambda s \geq 1$, hence leaving the question of the tightness of our bound open.

\section{Main results and global intuition} 
\subsection{Some definitions} 
Throughout the paper, some proposition $A_n$ is said to be true \emph{with high probability} (w.h.p.) if $\dP(A_n) \to 1$ when $n \to \infty$.

\paragraph{Finite sets, permutations}
For all $n>0$, we define $[n] := \left\{ 1, 2, \ldots, n \right\}$. For any finite set $\mathcal{X}$, we denote by $\card{\mathcal{X}}$ its cardinal. $\cS_{\mathcal{X}}$ is the set of permutations on $\mathcal{X}$. We also denote $\cS_{k} = \cS_{[k]}$ for brevity, and we will often identify $\cS_{k}$ to $\cS_{\mathcal{X}}$ whenever $\card{\mathcal{X}}=k$. For any permutations $\sigma,\sigma' \in \cS_n$ we denote by $\Fix(\sigma,\sigma')$ the number of points on which $\sigma=\sigma'$, namely 
\begin{equation*}
\Fix(\sigma,\sigma') := \sum_{i=1}^{n} \ind{\sigma(i)=\sigma'(i)}.
\end{equation*}

\paragraph{Graphs}
Through all the paper, we will implicitly consider that every graph $G$ of size $n$ has the canonical vertex set $[n]$. We will denote by $E(G)$ its edge set and $e(G)$ its number of edges.

For any pair of graphs $(G,G')$, both labeled on $[n]$, we denote by $G \lor G'$ (resp. by $G \land G'$) the union graph (resp. intersection graph) of $G$ and $G'$. The symmetric difference of $G$ and $G'$, denoted by $G \triangle G'$, is the subgraph made of edges of $G \lor G'$ that are not in $G \land G'$.

In the case where edges are colored (say $G$ is blue and $G'$ is red), these definitions extend to ensure colour preservation: note e.g. that in this case $G \land G'$ is simply the subgraph of $G \lor G'$ consisting of two-colored edges (see Figure \ref{fig:img_GunionGp}). 

\begin{figure}[H]
	\centering
	\includegraphics[scale=0.8]{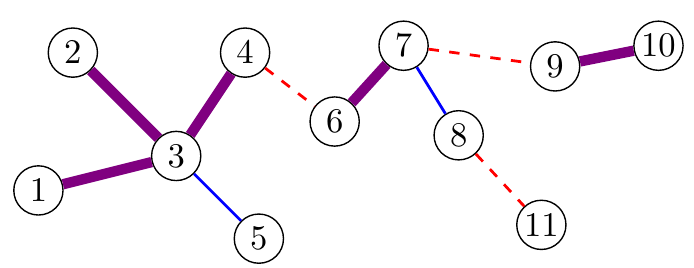}
	\caption{The graph $\cG \lor \cG'$ with $(\cG,\cG')$ of Figure \ref{fig:img_GGp}. For the sake of readability, the two-colored edges of $\cG \land \cG'$ are always drawn thick and purple.}
	\label{fig:img_GunionGp}
\end{figure} 

When the pair $(\cG,\cG')$ is drawn under the correlated \ER model, for all $u,v \in [n]$, we write $u \noir v$ (resp. $u \rouge v$) if $u$ and $v$ are connected in $\cG$, that is the edge is either blue or two-colored (resp. in $\cG'$, either red or two-colored).

For $G$ a graph with vertex set $[n]$ and $\sigma \in \cS_n$, we denote by $G^{\sigma}$ the \emph{relabeling of $G$ with $\sigma$}, which is the graph with same vertex set $[n]$ and edges $(\sigma(u), \sigma(v))$ for all $(u,v) \in E(G)$.

Finally we recall the definition of $c(\mu)$: for all $\mu > 0$, $c(\mu)$ is the greatest non-negative solution to the equation $e^{-\mu x}=1-x$. We also recall the fact that for $\mu \leq 1$, $c(\mu)=0$.

\subsection{General intuition on main result} 
Let us describe the general intuition for our result : recall that we are given $(\cG,\cH)$ drawn under the correlated \ER model with planted relabeling $\pi^*$. The idea of the argument for impossibility is to show that, there are w.h.p. lots of permutations that have the same weight for the posterior distribution of $\pi^*$ given $\cG, \cH$, and that are far apart. In other words, an informal statement is as follows :

\begin{informal*}
We want to show that there exists lots of relabelings $\cG^{\sigma_i}$ of $\cG$ such that:
\begin{itemize}
    \item[$(i)$] There is no way of deciding (statistically) whether the two graphs we observe are $(\cG, \cG')$ or some $(\cG^{\sigma_i}, \cG')$.
    \item[$(ii)$] These relabelings are far apart from each other and small components of $\cG \land \cG'$.
\end{itemize}
\end{informal*}

Let us give a formal version of the previous intuition. First note that for any labeled graphs $G,G'$ on $[n]$:
\begin{flalign*}
\dP(\cG=G,\cG'=G') 
& = \left(\frac{\lambda s}{n}\right)^{e(G \land G')} \left(\frac{\lambda (1-s)}{n}\right)^{e(G \triangle G')} \left(1 - \frac{\lambda (2-s)}{n}\right)^{\binom{n}{2} - e(G \lor G')}.
\end{flalign*}
Since $$e(G \lor G') = e(G) + e(G') - e(G \land G') \quad \mbox{and} \quad e(G \triangle G') = e(G \lor G') - e(G \land G'),$$  
$\dP(\cG=G,\cG'=G')$ is uniquely determined by $e(G), e(G')$ and $e(G \land G')$. In particular, the dependence of the joint distribution in $e(G \land G')$ is given by:

\begin{equation}\label{eq:joint_distribution}
\dP(\cG=G,\cG'=G')  \propto \left[\frac{s(n-\lambda(2-s))}{\lambda (1-s)^2}\right]^{e(G \land G')}.
\end{equation}

In view of \eqref{eq:joint_distribution}, preserving the posterior distribution by relabeling a graph $\cG$ is simply preserving the number of edges of their intersection graph. We now have a formal rephrasing for our conditions $(i)$ and $(ii)$ above: we encapsulate them in a theorem, which will constitute the bulk of our paper.

\begin{theorem}\label{theorem:autos} 
Fix an integer $p >0$. Consider $(\cG,\cG')$ drawn under the correlated \ER model. Then, with high probability,  there exists $\left\lbrace \sigma_i \right\rbrace_{i \in [p]}$ -- that depend on the intersection graph $\cG \land \cG'$ -- such that 
\begin{itemize}
\item[$(i)$] $\forall i \in [p], \; e\left(\cG^{\sigma_i} \land \cG'\right) = e\left(\cG \land \cG'\right)$,

\item[$(ii)$] $\forall i,j \in [p], \; i \neq j \implies \Fix(\sigma_i,\sigma_j) \leq c(\lambda s) n + o(n)$, where the $o(n)$ is independent of $i,j \in [p]$.
\end{itemize}
\end{theorem} Let us now explain how Theorem \ref{theorem:autos} implies our impossibility result via a simple pigeonhole principle. 

\begin{proof}[Proof of Theorem \ref{theorem:sparse_threshold}]
Let us take $\alpha>0$. We want to control the probability that the overlap between an estimator $\hat{\pi}$ and $\pi^*$ is greater than $\alpha n + c(\lambda s) n$. Fix $\eps>0$, and take $p$ large enough so that $$\alpha\eps  p >2.$$ 
First note that point $(i)$ together with \eqref{eq:joint_distribution} gives that the joint probability of $(\cG,\cG',\pi^*)$ is is equal to that of $(\cG^{\sigma_i},\cG',\pi^*)$, for all $i \in [p]$. Thus, for all estimator $\hat{\pi}$ depending on $\cG, \cH = \cG'^{\pi^*}$, one has
\begin{equation}\label{eq:egalite_distrib}
\forall i \in [p], \; \ov\left(\hat{\pi}(\cG^{\sigma_i}, \cH),\pi^*\right) \overset{(d)}{=} \ov\left(\hat{\pi}(\cG , \cH),\pi^*\right),
\end{equation} and by \eqref{eq:invariance_overlap}, we also have
\begin{equation}
\forall i \in [p], \; \ov\left(\hat{\pi}(\cG^{\sigma_i}, \cH),\pi^*\right) = \ov\left(\hat{\pi}(\cG, \cH),\pi^* \circ \sigma_i\right).
\end{equation}
Let
\begin{equation*}
X :=  \sum_{i \in [p]} \mathbf{1}_{\ov(\hat{\pi},\pi^* \circ {\sigma_i})>(c(\lambda s) + \alpha)n}
\end{equation*} Note that because of point $(ii)$, all $\Fix(\pi^* \circ {\sigma_i},\pi^* \circ {\sigma_j})$ are at most $c(\lambda s) n + o(n)$ for $i \neq j \in [p]$. Thus, there are at least $X \times (\alpha -o(1)) n$ distinct points among the node set $[n]$. This gives that one necessarily has 
\begin{equation}
X \leq \frac{1}{\alpha -o(1)}.
\end{equation}

Then, taking the expectation and considering the event on which the set $\left\lbrace \sigma_i \right\rbrace_{i \in [p]}$ of Theorem \ref{theorem:autos} exists -- which happens with probability $1-o(1)$ -- gives
\begin{flalign*}
\dE\left[X \right] & \geq \sum_{i=1}^{p} \dP\left(\ov(\hat{\pi},\pi^* \circ \sigma_i ) > (c(\lambda s) + \alpha)n  \right) - p \times o(1) \\
& = p \times \dP\left(\ov(\hat{\pi},\pi^*) > (c(\lambda s) + \alpha)n \right) - o(1).
\end{flalign*} Hence, 
\begin{equation}
\dP\left(\ov(\hat{\pi},\pi^*) > (c(\lambda s) + \alpha)n \right) \leq \frac{1}{p (\alpha -o(1))} + o(1).
\end{equation} 
For $n$ large enough, the right-hand side of the last term is less that $\frac{1}{p(\alpha/2)}$, which is less than $\eps$. This proves as desired that for all $\alpha >0$
\begin{equation}
\dP\left(\ov(\hat{\pi},\pi^*) > (c(\lambda s) + \alpha)n \right) \underset{n \to \infty}{\longrightarrow} 0.
\end{equation}
\end{proof}

We are now left to understand how to build ad hoc permutations verifying points $(i)$ and $(ii)$ of Theorem \ref{theorem:autos}. In order to build these permutations, we are going to relabel the vertices on small tree components of the intersection graph $\cG \land \cG'$. As a first step, we hereafter check that they indeed nearly cover the whole graph, when letting aside the giant component.

\subsection{Vertices on small tree components}
We briefly recall the definition of the simple \ER model $G(n,p)$: it consist in drawing a (single) graph with node set $[n]$ in which every edge is independently present with probability $p$. Let us begin with a classical result:

\begin{lemme}[\cite{Bollobas2001}, Corollary 5.8, Theorem 6.11]
	\label{lemma:bollobas_trees}
	Let $G\sim G(n,\mu/n)$ with $\mu>0$, and $a_n \to \infty$. Then, with high probability, $G$ has a giant component of order $c(\mu) n +o(n)$ and outside the giant component, at least $(1-c(\mu)) n-a_n$ vertices are on tree components.
\end{lemme} We need here a slight adaptation of this result, showing that $(1-c(\mu))n-o(n)$ vertices are in fact on \emph{small} tree components. 
\begin{lemme}
	\label{lemma:small_trees}
	Let $G\sim G(n,\mu/n)$ with $\mu >0$, and $K(n) \to \infty$. Then with high probability, $1-c(\mu))n-o(n)$ vertices are on tree components of size at most $K(n)$.
\end{lemme}
\begin{proof}
	Assume without loss of generality that $K(n) = o(\log n)$. Let $T_>$ be the number of vertices that are on tree components of size $\geq K(n)$. Taking $a_n = o(n)$ in Lemma \ref{lemma:bollobas_trees}, it remains to show that w.h.p., $T_> = o(n)$. This is done easily by bounding very roughly the first moment. Another classical result (see e.g. \cite{Janson00}, Theorem 5.4) is that with probability $1-o(1)$, all tree components are of size $O(\log n)$, which gives
	\begin{flalign*}
	\frac{\dE\left[T_>\right]}{n} &\leq o(1)+\sum_{k = K(n)}^{O(\log n)} \frac{1}{n} \cdot k \cdot \binom{n}{k} k^{k-2} \left(\frac{\mu}{n}\right)^{k-1} \left(1-\frac{\mu}{n}\right)^{k(n-k)+\binom{k}{2}-k+1}\\
	& \leq o(1)+ (1+o(1)) \sum_{k = K(n)}^{O(\log n)} \frac{e^k}{k} \mu^{k-1} e^{-k\mu},
	\end{flalign*} using $\binom{n}{k} \leq \left(\frac{en}{k}\right)^k$ together with Cayley's formula\footnote{Cayley's formula states that the number of trees on $k$ labeled vertices is $k^{k-2}$.} and the fact that for all indices $K(n) \leq k \leq O(\log n)$ in the sum, $k^2 \leq o(n)$ (uniformly). Now, the series in the right hand term has general terms which is $O\left(e^{-k(\mu - \log \mu +1)}\right)$, and since $\mu - \log \mu +1>0$ the series converges, which implies that $\dE\left[T_>\right/n]=o(1)$. The proof is concluded by Markov's inequality.
\end{proof}

Since in our model $\cG \land \cG'$ is an \ER graph of parameters $(n,\lambda s/n)$, the previous results ensures that all but a vanishing part of the $(1-c(\mu)) n$ vertices outside the giant component are on small (i.e. $\leq K(n)$) tree components of the intersection graph. For the rest of the paper, we will take $$K(n) = \lfloor \sqrt{\log n} \rfloor.$$

This first step suggests to build the permutations (relabelings) only by looking at $\cG \land \cG'$. Hence, we will first consider the random generation of the intersection graph, then create some permutations $\sigma_i$, and finally reveal the monochromatic edges.

The generating process is as follows: since almost all  $(1-c(\mu)) n$ vertices are on small trees in $\cG \land \cG'$, we can prove that each small tree up to isomorphism will have a number of occurrences in the intersection graph of order $n$ (this is claimed more precisely in Lemma \ref{lemme:controle_X}). Permuting iteratively these isomorphic trees, we may derange them quite a lot, and each time differently. 

In order to prove Theorem \ref{theorem:autos}, we use the \emph{probabilistic method}\footnote{The main interest of this widely used method (see \cite{ProbaMethod}) is to be non-constructive. Indeed, as detailed in the next Sections, explicitly giving the $p$ permutations considered in Theorem \ref{theorem:autos} is very cumbersome, because of the extra double edges that may appear (see Section \ref{section_extra_double_edges}).}: we give in the next section a simple detailed stochastic method to build $p$ permutation candidates, and we will next prove that these permutations satisfy conditions $(i)$ and $(ii)$ with positive probability, hence proving the desired existence.

\section{Building automorphisms of $\cG \land \cG'$ tree-wise}
Through all this section, we work conditionally on the intersection graph $\cG \land \cG'$ (that is the two-colored edges). 

\subsection{Mathematical formalization}
Recall that we fix $K:=K(n)= \lfloor \sqrt{\log n} \rfloor$. For all $k \in [K]$, we will denote by $\dT_k$ the set of \emph{unlabeled} trees of size $k$. $\dT_k$ can also be viewed as the set of equivalence classes of labeled trees of size $k$ for the isomorphism relation. Note that $\dT_k$ is finite and that we can roughly upper bound its size by the number of \emph{labeled} trees of size $k$ which equals $k^{k-2}$, by Cayley's formula\footnote{This upper bound is far from being optimal, but is enough for our use.}.

\begin{figure}[h]
	\centering
	\includegraphics[scale=1]{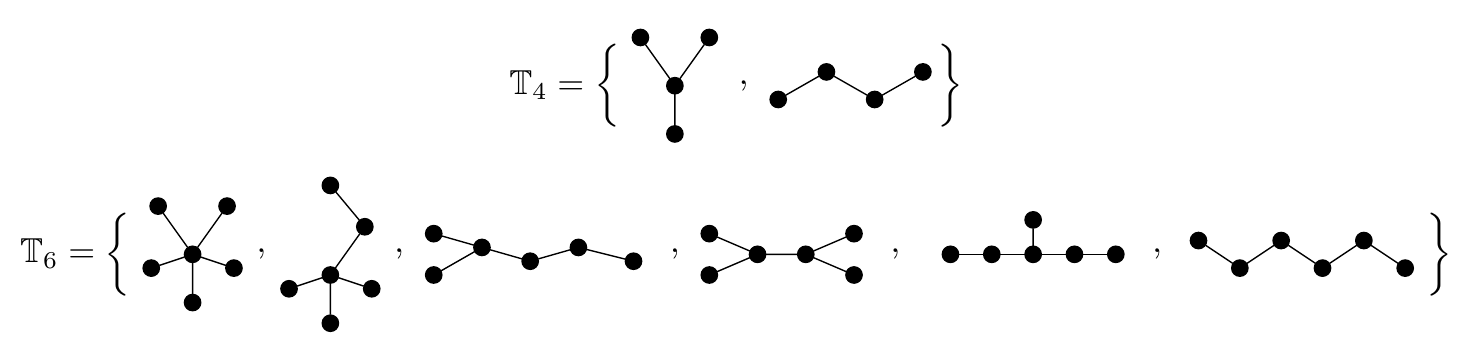}
	\caption{Explicit composition of $\dT_4$ (of size $2$) and $\dT_6$ (of size $6$).}
	\label{fig:img_ex_Tk}
\end{figure} 

For a given tree $\bT \in \dT_k$, we will denote by $X_\bT$ the number of distinct connected components of $\cG \land \cG'$ that are isomorphic to $\bT$, $H_\bT := \left\lbrace \cT_1, \cT_2, \ldots, \cT_{X_\bT} \right\rbrace$ the set of the corresponding labeled subgraphs of $\cG \land \cG'$, and $V(H_\bT)$ the set of vertices of $[n]$ that belong to one of the trees in $H_\bT$.

Our global recursion will be done on the finite set 
\begin{equation}\label{eq:def_T}
\dT := \bigcup_{k=1}^{K} \dT_k = \left\{\bT_1, \bT_2, \ldots, \bT_M\right\},
\end{equation} which we assume to have been ordered increasingly according to tree sizes, for convenience. The global permutation $\sigma$ is built block-wise by composing permutations $\sigma_\bT$ for $\bT \in \dT$ such that each $\sigma_\bT$ only acts on vertices of $H_\bT$.

More precisely, for a fixed $\bT \in \dT$, $\sigma_\bT$ will consists in permuting the vertices tree by tree, so $\sigma_\bT$ will be determined by a tree permutation $\Sigma_\bT$ of size $X_\bT$. Assume that for all trees $\cT_1, \ldots, \cT_{X_\bT}$ isomorphic to $\bT$ in $\cG \land \cG'$, we fix some isomorphisms $\psi_1, \ldots, \psi_{X_\bT}$ such that $\cT_i\underset{\psi_i}{\sim} \bT$ for all $i \in [X_\bT]$. More generally we will denote $\itree(u)$ the index of the tree that $u \in V(H_\bT)$ belongs to (when there is no ambiguity on $\bT$), and $u \simt u'$ when two vertices of $\cG \land \cG'$ are sent onto the same point of $\bT$ by these isomorphisms. Then, the natural definition of the node permutation $\sigma_{\bT}$ according to $\Sigma_\bT$ and these isomorphisms is given by

\begin{equation}\label{eq:Sigma_sigma}
\sigma_\bT : u \mapsto 
\left\{
\begin{array}{ll}
\psi_{\Sigma_\bT (\itree(u))}^{-1} \circ \psi_{\itree(u)} (u) \;\; (\in \cT_{\Sigma_\bT (\itree(u))}) & \mbox{if } u \in V(H_{\bT}), \\
u & \mbox{if } u \notin V(H_{\bT}).
\end{array}
\right.
\end{equation} Note that by definition, $V(H_{\bT})$ is stable by $\sigma_\bT$, and $\sigma_\bT$ fixes all nodes in $[n] \setminus V(H_{\bT})$. Recall that $M$ denotes the total size of $\dT$ as defined in \eqref{eq:def_T}. The recursive construction is as follows :

\begin{algorithm}[H]
	\label{algo:rec_construction}
	\caption{Recursive construction of $\sigma$}
	\SetAlgoLined
	Initialize $\sigma_0 \gets \id$\;
	
	\For{$i=1$ to $M$}{
		Consider $\bT = \bT_i$ and draw uniformly at random the tree permutation $\Sigma_{\bT} \in \cS_{X_{\bT}}$, independently from the past\;
		
		Consider $\sigma_{\bT}$ the node permutation associated with $\Sigma_{\bT}$ by \eqref{eq:Sigma_sigma}\;
		
		$\sigma_{i} \gets \sigma_{\bT} \circ \sigma_{i-1}$\;
	}
	\textbf{return} $\sigma = \sigma_M$
\end{algorithm} Note that at the end of the procedure, $\sigma$ fixes all points that are either on the giant component of the intersection graph, or on a component that is not a tree a size $\leq K(n)$. Figure \ref{fig:img_algo_ok} gives an example of this random recursive construction (for convenience, $\lambda s <1$, the true labels are in red, whereas blue labels enables to keep track of the relabeling recursively built on the blue graph). 

\begin{figure}[h]
	\centering
	\includegraphics[scale=0.83]{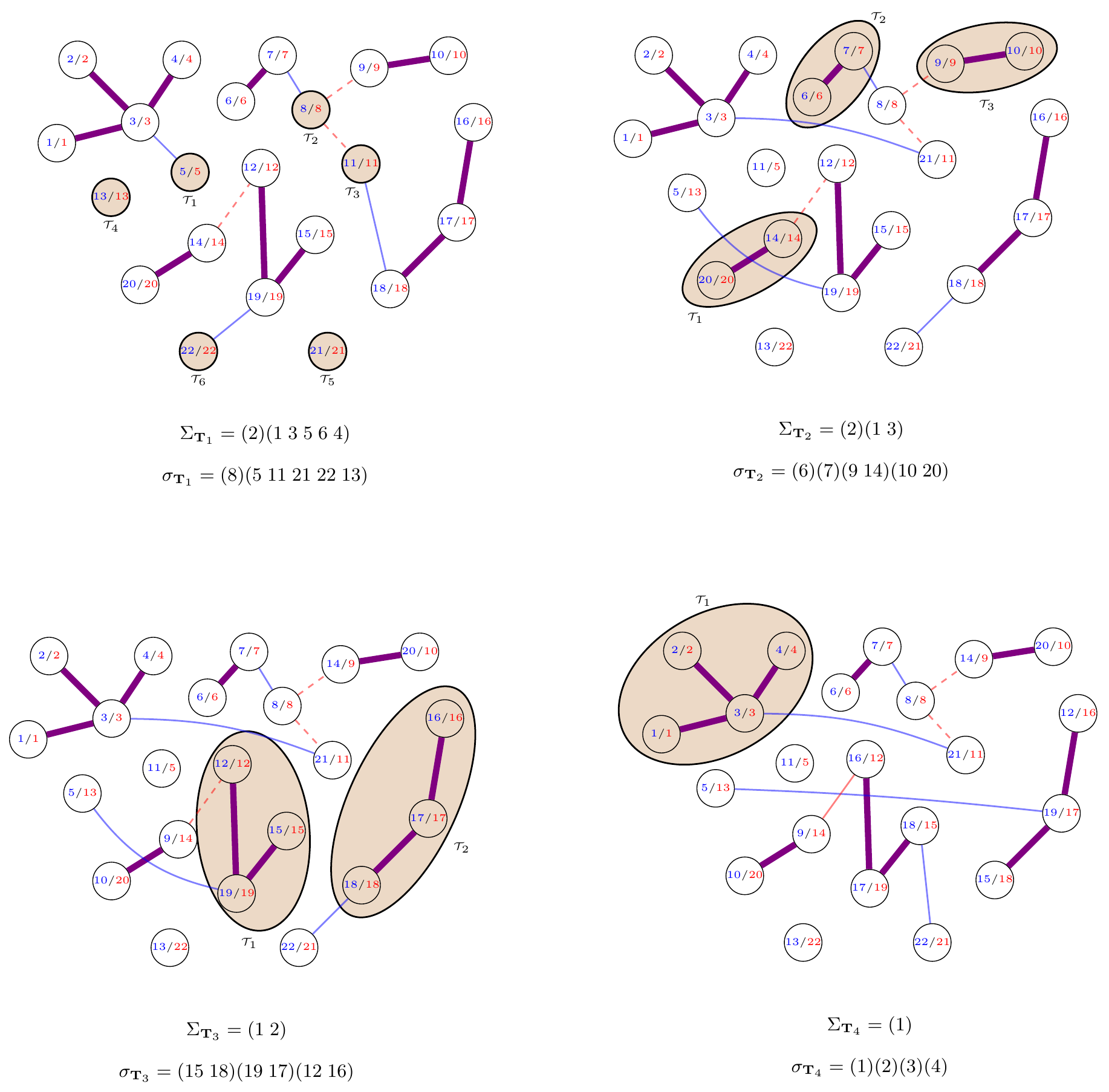}
	\caption{Example of recursive (tree-wise) generation of a permutation with Algorithm \ref{algo:rec_construction}.}
	\label{fig:img_algo_ok}
\end{figure} 

Through the analysis we will need the following control on $X_\bT$ for $\bT \in \dT$:

\begin{lemme}\label{lemme:controle_X}
Recall that $K(n)= \lfloor \sqrt{\log n} \rfloor$. For all $k \in [K(n)]$, define $
f(k) := \frac{(\lambda s)^{k-1}e^{-\lambda s k}}{k!}$. Then, with high probability (on the intersection graph),
\begin{equation}\label{eq:controle_X}
\forall k \in [K(n)], \, \forall \bT \in \dT_k, \, X_\bT \geq n (1-o(1)) f(k).\end{equation}
\end{lemme} The proof of this result is deferred to Appendix \ref{app:proof_control_X}.

\begin{remarque}\label{remarque:f(K)}
Note that since $\lambda s e^{-\lambda s}<1$, $k \mapsto f(k)$ is decreasing with $k$. Moreover, for $K(n) \leq \sqrt{\log n}$, we have that for any $t>0$, $$f(K(n)) \geq \exp\left(- C\sqrt{\log n} \log \log n \right) \gg n^{-t}.$$  
\end{remarque}

\subsection{Ensuring that the permutations are 'far apart'} 
We check in this section that Algorithm \ref{algo:rec_construction} generates permutations that will verify condition $(ii)$ of Theorem \ref{theorem:autos}, w.h.p. Let $\sigma_{1}, \ldots, \sigma_{p}$ be generated independently with Algorithm \ref{algo:rec_construction}. We then have the following results:

\begin{lemme}\label{lemma:controle_overlap_ij} With high probability, for all $i \neq j \in [p]$,
$$ \Fix(\sigma_i,\sigma_j) = c(\lambda s) n + o(n).$$
\end{lemme}

This lemma is proved in Appendix \ref{app:proof_overlap}. In the sequel we will denote by $V_\infty$ the set of vertices that are on the giant component of $\cG \land \cG'$ (if there is one), and by $V_>$ the vertices of $[n] \setminus V_\infty$ that are \emph{not} on tree components of size $\leq K(n)$. Finally we set $V_{\infty,>} := V_\infty \cup V_>$. Define
\begin{equation}\label{eq:def_Sin_Sout}
    \cS_{in} := \binom{[n]\setminus V_{\infty,>} }{2}, \quad \cS_{out} := \binom{[n]}{2} \setminus \left(\binom{V_{\infty,>}}{2} \cap \binom{[n]\setminus V_{\infty,>}}{2} \right), \quad \cS := \cS_{in} \cup \cS_{out}.
\end{equation}$\cS_{in}$ is the set of edges that have both endpoints outside $V_{\infty,>}$, whereas edges of $\cS_{out}$ have exactly one endpoint in $V_{\infty,>}$. We say that an edge $(u,v) \in \binom{[n]}{2}$ is a \emph{common fixed edge} of permutations $\sigma_{1}, \ldots, \sigma_r$ if 
$$ \left\lbrace \sigma_1(u), \sigma_1(v) \right\rbrace = \ldots = \left\lbrace \sigma_r(u), \sigma_r(v) \right\rbrace. $$ For all subset of edges $\cW \subseteq \binom{[n]}{2}$, we define
\begin{equation}
F(\cW,\sigma_{1}, \ldots, \sigma_r) := \sum_{e \in \cW} \mathbf{1}_{e \mbox{\footnotesize{ is a common fixed edge of} } \sigma_{1}, \ldots, \sigma_r}.
\end{equation} 

We now state a result -- which proof is deferred to \ref{app:proof_control_F} -- that will be useful in next section.

\begin{lemme}\label{lemma:controle_F}
With high probability, we have, for any $t>0$,
\begin{itemize}
    \item for any $i_1 \neq i_2 \in [p]$,
    \begin{equation}\label{eq:lemma:controle_F2}
	F(\cS,\sigma_{i_1},\sigma_{i_2}) \leq n^{1+t},
	\end{equation}
	
    \item for any $i_1, i_2, i_3 \in [p]$ pairwise distinct,
    \begin{equation}\label{eq:lemma:controle_F3}
	F(\cS,\sigma_{i_1},\sigma_{i_2},\sigma_{i_3}) \leq n^t,
	\end{equation}
	
	\item for any $r \geq 4$, $i_1, \ldots, i_r \in [p]$ pairwise distinct,
	\begin{equation}\label{eq:lemma:controle_F4}
	F(\cS,\sigma_{i_1},\ldots,\sigma_{i_r}) =0.
	\end{equation}
\end{itemize}
\end{lemme}

\subsection{Emergence of extra double edges}\label{section_extra_double_edges} 
In the example of Figure \ref{fig:img_algo_ok}, we can see that the number of two-colored edges in the relabeled union graph $\cG^{\sigma_i} \lor \cG'$ is constant through time. This property is fundamental for point $(i)$ of Theorem \ref{theorem:autos}. However, depending on the random $\sigma_{\bT_i}$ drawn through the process -- we recall that they are drawn independently from the monochromatic edges, that are not revealed yet -- we may see extra two-colored edges appear (extra double edges hereafter). Figure \ref{fig:img_algo_pb} shows a case in which there is an emergence of an extra double edge in the process.

\begin{figure}[H]
	\centering
	\includegraphics[scale=0.85]{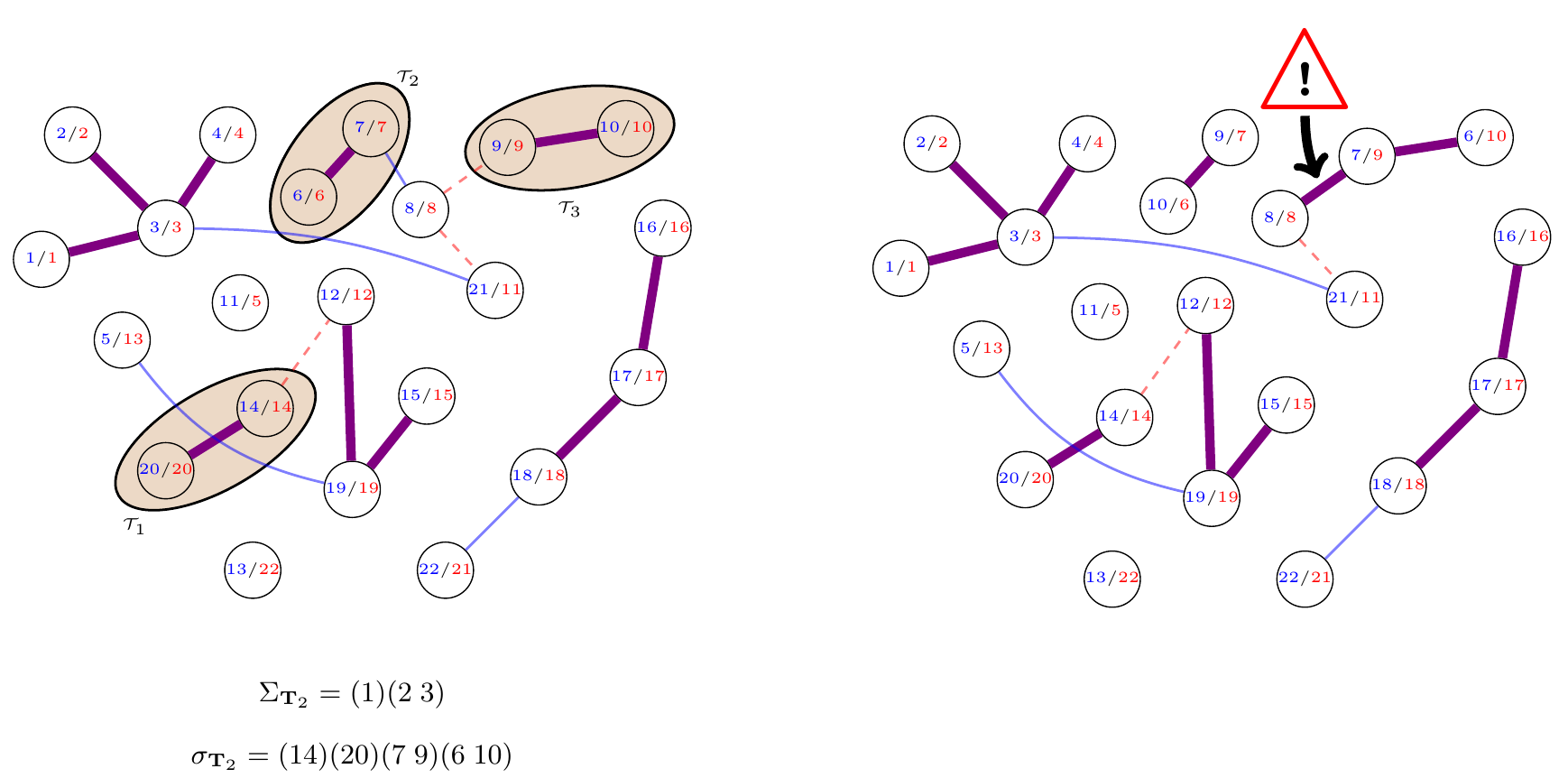}
	\caption{Example of the emergence of an extra double edge in Algorithm \ref{algo:rec_construction}.}
	\label{fig:img_algo_pb}
\end{figure}
Note that the number of two-coloured edges can only be greater or equal to $e(\cG \land \cG')$ through this process, since by definition we are preserving edges of the intersection graph.

The last part of our work is to prove that there is a positive probability that applying independently Alg.\ref{algo:rec_construction} $p$ times gives $p$ permutations that do not present extra double edges, before using the probabilistic method. This step will require a Poisson approximation, described hereafter.

\section{Poisson approximation to avoid extra double edges, proof of Theorem \ref{theorem:autos}.}

In this section we introduce $n'$ to be the number of vertices that the permutations actually act on:
\begin{equation}
n' := \card{[n] \setminus V_{\infty,>}} \sim (1-c(\lambda s)) n \mbox{ w.h.p.}
\end{equation} Then, we assume that we fix a set $\left\{\sigma_{i}\right\}_{i \in [p]}$ of $p$ permutations of $[n']$, verifying : 
\begin{equation}\tag{H1}\label{H1}
\mbox{for all } t>0, \mbox{for all } m\neq m' \in [p], F(\cS,\sigma_{m},\sigma_{m'}) \leq n^{1+t}.
\end{equation} 
\begin{equation}\tag{H2}\label{H2}
\mbox{for all } t>0, \mbox{for all } m_1, m_2, m_3 \in [p] \mbox{ pairwise distinct }, F(\cS,\sigma_{m_1},\sigma_{m_2},\sigma_{m_3}) \leq n^{t}.
\end{equation} 
\begin{equation}\tag{H3}\label{H3}
\mbox{There are no common fixed edge of any $r$-tuple in $\left\{\sigma_{i}\right\}_{i \in [p]}$.}
\end{equation}

We will work under the event $\cE_\cS$ on which $n' \sim (1-c(\lambda s)) n$ and $\card{\cS} \sim \binom{n'}{2} \sim n'^2/2 = (1-c(\lambda s))^2 n^2/2 $. It is easy (see e.g. \cite{Bollobas2001}) to show that $\cE_\cS$ is satisfied w.h.p. As explained before, some extra double edges (e.d.e. hereafter) may appear when revealing the non double edges of $\cS$ (that is, blue and red edges that are not between vertices of $V_{\infty,>}$). Note that for every edge we have
\begin{flalign*}
\dP\left(u \noir v \, | \, (u,v) \notin E(\cG \land \cG') \right) &= \dP\left(u \rouge v \, | \, (u,v) \notin E(\cG \land \cG') \right)\\
&= \frac{\dP\left(u \rouge v, (u,v) \notin E(\cG \land \cG') \right)}{\dP\left((u,v) \notin E(\cG \land \cG') \right)} \\ &= \frac{\lambda(1-s)/n}{1 - \lambda s/n} \sim \frac{\lambda(1-s)}{n}.
\end{flalign*}
For any permutation $\sigma$, define the number of created e.d.e. by the relabeling of $\cG$ by $\sigma$ as follows:
\begin{equation}\label{def:delta_sigma}
\Delta(\sigma) := \sum_{\left\{u,v\right\} \in \cS} \mathbf{1}_{u \nnoir v} \mathbf{1}_{\sigma(u) \rrouge \sigma(v)}.
\end{equation}
We now present the key result for our analysis, with the notation $n^{\underline{k}}$ for the \emph{falling factorial} $$n^{\underline{k}}:= n(n-1)\cdots(n-k+1).$$

\begin{theorem}[Asymptotic Poisson behavior of $\left\{\Delta(\sigma_{i})\right\}_{i \in [p]}$]\label{theorem:poisson}
Assume that $\left\{\sigma_{i}\right\}_{i \in [p]}$ verify \eqref{H1}, \eqref{H2} and \eqref{H3}. Then, for all $\ell_1, \ell_2, \ldots, \ell_p \geq 0$,
\begin{equation}\label{eq:theorem:poisson}
\dE \left[\Delta(\sigma_1)^{\underline{\ell_1}}\Delta(\sigma_2)^{\underline{\ell_2}} \cdots \Delta(\sigma_p)^{\underline{\ell_p}} \, \big| \, \cG \land \cG', \cE_\cS\right] \underset{n \to \infty}{\longrightarrow} \left(\frac{\lambda^2(1-s)^2 (1-c(\lambda s))^2}{2}\right)^{\ell_1 + \ell_2 + \ldots + \ell_p}.
\end{equation} In other words, conditionally to graph $\cG \land \cG'$ and event $\cE_\cS$, the random variables $\left\{\Delta(\sigma_{i})\right\}_{i \in [p]}$ are asymptotically distributed as independent Poisson variables of parameter $\frac{\lambda^2(1-s)^2 (1-c(\lambda s))^2}{2}$.
\end{theorem}

The proof of Theorem \ref{theorem:poisson}, based on a fine control of terms of unusually high contribution, is deferred to Appendix \ref{app:proof_theorem_poisson}.
\subsection{Proof of Theorem \ref{theorem:autos}}
\begin{proof}
The proof is quite straightforward now. Fixing $p>0$, Lemma \ref{lemma:controle_F} gives that \eqref{H1}, \eqref{H2} and \eqref{H3} are verified w.h.p. by some $\sigma_{1}, \ldots, \sigma_{p}$ generated independently with Algorithm \ref{algo:rec_construction}. Then, the probability (on the remaining monochrome edges) that the $p$ permutations given satisfy conditions $(i)$ and $(ii)$ of Theorem \ref{theorem:autos} is equivalent to
\begin{equation}
(1-o(1)) \times \dP\left(\Poi\left(\frac{\lambda^2(1-s)^2}{2}\right)=0 \right)^p = (1-o(1)) \exp\left(- p \frac{\lambda^2(1-s)^2}{2}\right) >0,
\end{equation} which gives the existence with high probability of a set a permutations of size $p$ satisfying conditions $(i)$ and $(ii)$ of Theorem \ref{theorem:autos}.
\end{proof}

\newpage

\acks{This work was partially supported by the French government under management of Agence Nationale de la Recherche as part of the “Investissements d’avenir” program, reference ANR19-P3IA-0001 (PRAIRIE 3IA Institute).}

\bibliographystyle{alpha}
\bibliography{biblio_GA}

\newpage
\appendix
\section{Proof of Theorem \ref{theorem:poisson}}\label{app:proof_theorem_poisson}
\begin{proof}[Proof of Theorem \ref{theorem:poisson}]
Let $\ell_1, \ell_2, \ldots, \ell_p $ be non negative integers. Recall that conditioned to $\cG \land \cG'$, each edge of $\cS$ is independently blue (resp. red) with probability 
\begin{equation*}
q = q(\lambda,s,n) := \frac{\lambda(1-s)}{n - \lambda s}.
\end{equation*}
Now, let us explain why convergence \eqref{eq:theorem:poisson} holds. First recall that for a given $\ell \geq 0$, $\dE\left[\Delta(\sigma)^{\underline{\ell}}\right]$ is nothing else but the expected number of (ordered) $p-$tuples of edges $\left\{u,v\right\} \in \cS$ such that $\mathbf{1}_{u \nnoir v} \mathbf{1}_{\sigma(u) \rrouge \sigma(v)} = 1$. Using the notation ${\sum}^*$ for summation of ordered tuples of edges in $\cS$ as well as linearity of expectation, we get:

\begin{multline}\label{eq:expr_multi_mom}
\dE \left[\Delta(\sigma_1)^{\underline{\ell_1}}\Delta(\sigma_2)^{\underline{\ell_2}} \cdots \Delta(\sigma_p)^{\underline{\ell_p}}\right] = \\
\sideset{}{^*}\sum_{\substack{\lbrace u^{(1)}_{1},v^{(1)}_{1}\rbrace, \\ \lbrace u^{(1)}_{2},v^{(1)}_{2}\rbrace, \\\ldots, \\ \lbrace u^{(1)}_{\ell_1},v^{(1)}_{\ell_1}\rbrace} } \; \;
\sideset{}{^*}\sum_{\substack{\lbrace u^{(2)}_{1},v^{(2)}_{1}\rbrace, \\ \lbrace u^{(2)}_{2},v^{(2)}_{2}\rbrace, \\\ldots, \\ \lbrace u^{(2)}_{\ell_2},v^{(2)}_{\ell_2}\rbrace} }
\ldots
\sideset{}{^*}\sum_{\substack{\lbrace u^{(p)}_{1},v^{(p)}_{1}\rbrace, \\ \lbrace u^{(p)}_{2},v^{(p)}_{2}\rbrace, \\\ldots, \\ \lbrace u^{(p)}_{\ell_p},v^{(p)}_{\ell_p}\rbrace} }
\dE \left[\prod_{m=1}^{p} \prod_{j=1}^{\ell_m} \mathbf{1}_{u^{(m)}_{j} \nnoir v^{(m)}_{j}} \mathbf{1}_{\sigma_m(u^{(m)}_{j}) \rrouge \sigma_m(v^{(m)}_{j})}\right] 
\end{multline}

First observe that the total number of terms $N$ in the previous sum is 
$$N := \card{\cS}^{\underline{\ell_1}} \times \card{\cS}^{\underline{\ell_2}} \times \cdots \card{\cS}^{\underline{\ell_p}} \sim \left(\frac{(1-c(\lambda s))^2 n^2 }{2}\right)^{\ell_1 + \ldots + \ell_p},$$ since $\card{\cS} \sim \frac{ (1-c(\lambda s))^2 n^2}{2}$ on event $\cE_\cS$.\\

\textbf{Lower bound:} Observe that the $N$ terms in the sum of eq. \eqref{eq:expr_multi_mom} are made in general of $2(\ell_1 + \ldots + \ell_p)$ indicator variables, not necessarily distinct. For most of the terms however, all involved edges are distinct, thus independent, and their contribution to the sum is $q^{2(\ell_1 + \ldots + \ell_p)}$.

Whenever a pair of blue (resp. red) indicators are equal, at least one term may be canceled, so the contribution to the expectation is higher than $q^{2(\ell_1 + \ldots + \ell_p)}$.

Whenever a pair of edges that appear in a blue/red pair of indicators are equal, the product of the indicators is necessarily $0$ (indeed, an edge in $\cS$ cannot be two-colored). These terms, where at least one equality of the form $\left\{u_j^{(m)}, v_j^{(m)}\right\}  = \left\{\sigma_{m'}(u^{(m')}_{j'}) , \sigma_{m'}(v^{(m')}_{j'})\right\} $ occurs, cover the case where the contribution is strictly less that $q^{2(\ell_1 + \ldots + \ell_p)}$ (it is $0$). There are at most 
$$\binom{\ell_1+\ldots+\ell_p}{2} \left(\frac{n^2}{2}\right)^{\ell_1+\ldots+\ell_p-1}$$ such terms. Thus

\begin{flalign*}
\dE \left[\Delta(\sigma_1)^{\underline{\ell_1}}\Delta(\sigma_2)^{\underline{\ell_2}} \cdots \Delta(\sigma_p)^{\underline{\ell_p}}\right] & \geq \left(N - \binom{\ell_1+\ldots+\ell_p}{2} \left(\frac{n^2}{2}\right)^{\ell_1+\ldots+\ell_p-1}\right) \times q^{2(\ell_1 + \ldots + \ell_p)}\\
& \sim \left(\frac{(1-c(\lambda s))^2 n^2}{2}\right)^{\ell_1 + \ldots \ell_p} \times \left(\frac{\lambda (1-s)}{n}\right)^{2(\ell_1 + \ldots + \ell_p)} \\
&\underset{n \to \infty}{\longrightarrow} \left(\frac{\lambda^2(1-s)^2 (1-c(\lambda s))^2}{2}\right)^{\ell_1 + \ell_2 + \ldots + \ell_p}.
\end{flalign*}

\textbf{Upper bound:} The terms that we now want to study are the terms for which the contribution is greater than $q^{2(\ell_1 + \ldots + \ell_p)}$. Looking closely at the general product in \eqref{eq:expr_multi_mom}, an unusual high contribution is the consequence of three possible type of constraints:
\begin{itemize}
\item[$(i)$] constraints of the form  $\left\{u_j^{(m)}, v_j^{(m)}\right\}  = \left\{u^{(m')}_{j'} , v^{(m')}_{j'}\right\} $: note that since the sums are made of ordered tuples, this equality may happen only for pairs such that $m \neq m'$. Moreover, transitivity of equality implies that a constraint implying some $\left\{u_j^{(m)}, v_j^{(m)}\right\}$ may happen at most once for each $m' \in [p], m' \neq m$ (otherwise we would have a relationship of the form $\left\{u_{j'}^{(m')}, v_{j'}^{(m')}\right\} = \left\{u_{k'}^{(m')}, v_{k'}^{(m')}\right\}$, which is impossible).

\item[$(ii)$] constraints of the form  $\left\{\sigma_{m}(u^{(m)}_{j}) , \sigma_{m}(v^{(m)}_{j})\right\}  = \left\{\sigma_{m'}(u^{(m')}_{j'}) , \sigma_{m'}(v^{(m')}_{j'})\right\} $. For the same reasons as in case $(i)$, a constraint implying some $\left\{\sigma_{m}(u^{(m)}_{j}) , \sigma_{m}(v^{(m)}_{j})\right\}$ may happen at most once for each $m' \in [p], m' \neq m$.

\item[$(iii)$] the last case is made of intersection of cases $(i)$ and $(ii)$, i.e. edges satisfying both constraints $\left\{u_j^{(m)}, v_j^{(m)}\right\}=\left\{u^{(m')}_{j'}, v^{(m')}_{j'}\right\}$ and $\left\{\sigma_{m}(u^{(m)}_{j}) , \sigma_{m}(v^{(m)}_{j})\right\}=\left\{\sigma_{m'}(u^{(m')}_{j'}) , \sigma_{m'}(v^{(m')}_{j'})\right\}$. This implies in particular that $\left\{u_j^{(m)}, v_j^{(m)}\right\}$ is an common fixed edge for $\sigma_{m}$ and $\sigma_{m'}$. By assumption \eqref{H3}, note that there cannot be a connected path of constraints of the form $(iii)$ of length greater or equal to $3$.
\end{itemize}

Let us now represent these constraints with a dependency graph. Each vertex a the graph represent one edge $\left\{u_j^{(m)}, v_j^{(m)}\right\}$ of the sum, that we will align column-wise according to $m \in [p]$. We put a plain (resp. dashed) edge between two nodes if they are enforced by constraint $(i)$ but not $(iii)$ (resp. $(ii)$ but not $(iii)$). Finally we draw a thick plain edge between two nodes if they are enforced by constraint $(iii)$.

In view of discussion in points $(i)-(ii)-(iii)$, this dependency graph must be $p$-partite. Moreover, the subgraph made of plain thick or plain edges (resp. plain thick of dashed edges) only consists in a union of disjoint paths. The thick plain subgraph is only made of isolated edges and paths fo size $3$. Finally, transitivity of the equality relationship enables to draw any path in any order: we shall take the left to right order by convention (no backtracking). 

We denote by $k_1$ (resp. $k_2$) the number of plain (resp. dashed) edges. We also denote track $k_3$ the number of thick plain isolated edges, and $k_4$ the number of thick plain isolated paths of length $2$. Figure \ref{fig:img_dependency} gives an example of such a dependency graph.

\begin{figure}[H]
	\centering
	\includegraphics[scale=0.8]{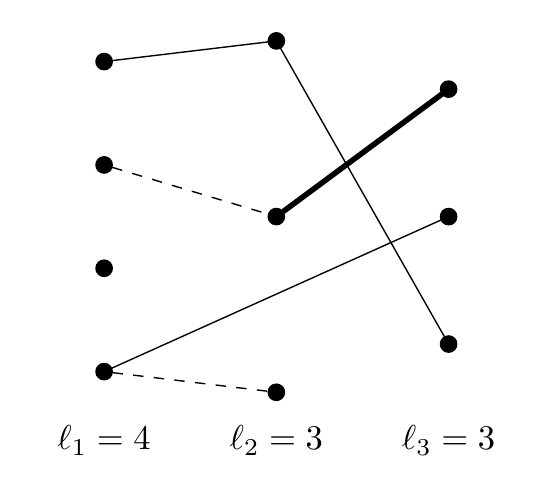}
	\caption{Example of a dependency graph, with $(k_1,k_2,k_3,k_4) = (3,2,1,0)$.}
	\label{fig:img_dependency}
\end{figure} 

In order to upper bound the contribution due to large terms, we must understand both the expectation of the product of indicators in \eqref{eq:expr_multi_mom} (this only depends on $(k_1,k_2,k_3,k_4)$), as well as the number of possible (labeled) dependency graphs with a given $(k_1,k_2,k_3,k_4)$.

First, all plain (resp. dashed) dependency edge makes $1$ (resp. $1$) indicators disappear in the expectation (for any event $\cA, \mathbf{1}_\cA^2 = \mathbf{1}_\cA$). In the same way, all thick plain isolated edge (resp. thick plain isolated path of length $2$) makes $2$ (resp. $4$) indicators disappear
the expectation for a given case with given $(k_1,k_2,k_3,k_4)$ is 
\begin{equation}\label{eq:contrib_cas_patho}
q^{2(\ell_1 + \ldots + \ell_p) - (k_1+k_2+2k_3+4k_4)} \leq C_1 n^{-2(\ell_1 + \ldots + \ell_p) + (k_1+k_2+2k_3+4k_4)}
\end{equation} where $C_1$ is a constant depending on $\ell_1, \ldots, \ell_p$,

Second, an upper bound for the number of possible (labeled) dependency graphs with a given $(k_1,k_2,k_3,k_4)$ can be established as follows. First, we have $k_1+k_2+k_3+2k_4$ equalities, leaving at most $\ell_1+\ldots+\ell_p - (k_1+k_2+k_3+2k_4)$ degrees of freedom in the choices of the edges. Moreover, we force $k_3$ of these edges to be common fixed edges between two (distinct) permutations, and $k_4$ of them to be common fixed edges between three (pairwise distinct) permutations. In view of hypotheses \eqref{H1} and \eqref{H2}, the number of possible (labeled) dependency graphs with a given $(k_1,k_2,k_3,k_4)$ is at most 

\begin{flalign}\label{eq:nombre_cas_patho}
\binom{k_1+k_2+k_3+k_4}{k_3+k_4}\card{\cS}^{\ell_1+\ldots+\ell_p - (k_1+k_2+k_3+2k_4) - k_3 -k_4} \times (n^{1+t})^{k_3} \times n^{t k_4} \nonumber \\ \leq C_2 n^{2(\ell_1+\ldots+\ell_p) - 2(k_1+k_2)-(3-t)k_3 -(6-t)k_4},
\end{flalign} where $C_2$ is a constant depending on $\ell_1, \ldots, \ell_p$.

Hence, in view of \eqref{eq:contrib_cas_patho} and \eqref{eq:nombre_cas_patho}, the total contribution of higher terms is upper bounded by

\begin{flalign*}
& \sum_{s=1}^{\ell_1+\ldots+\ell_p} \sum_{k_1+k_2+k_3+2k_4=s} C_1 C_2 n^{-2(\ell_1 + \ldots + \ell_p) + (k_1+k_2+2k_3+4k_4)} n^{2(\ell_1+\ldots+\ell_p) - 2(k_1+k_2)-(3-t)k_3 -(6-t)k_4}\\
&\leq C_1 C_2 \sum_{s=1}^{\ell_1+\ldots+\ell_p} \sum_{k_1+k_2+k_3+2k_4=s} n^{-k_1} n^{-k_2} n^{-(1-t)k_3} n^{-(2-t)k_4} \\
&\leq C_1 C_2 \times (\ell_1+\ldots+\ell_p)\times (\ell_1+\ldots+\ell_p)^{4 (\ell_1+\ldots+\ell_p)} \times n^{-(1-t)} \underset{n \to \infty}{\longrightarrow} 0.
\end{flalign*} This last convergence concludes the proof.
\end{proof}

\section{Proofs of Lemmas}
\subsection{Proof of Lemma \ref{lemme:controle_X}}\label{app:proof_control_X}
\begin{proof}
	For the control of $X_\bT$ we follow classical computations made in \cite{Bollobas2001} to establish asymptotic behavior of $X_\bT$. For our purpose, we only need the two first moments. Assume that $\bT$ is of size $k = k(\bT) \leq K$, and that its automorphism group has $a =a(\bT)$ elements. Then, letting $\mu = \lambda s$,
	\begin{flalign*}
	\dE\left[X_\bT\right] &= \binom{n}{k} \times \frac{k!}{a} \times \left(\frac{\mu}{n}\right)^{k-1} \left(1-\frac{\mu}{n}\right)^{k(n-k)+\binom{k}{2}-k+1}.
	\end{flalign*} Indeed, we have $\binom{n}{k}$ choices for the nodes, then $\frac{k!}{a}$ ways of putting the edges. 
	Using $ \binom{n}{k} \sim \frac{n^k}{k!}$ and $\left(1-\frac{\mu}{n}\right)^{-k^2+\binom{k}{2}-k+1} \sim 1$ as soon as $k = o(\sqrt{n})$, we get
	\begin{flalign*}
	\dE\left[X_\bT\right] &\sim n \mu^{k-1}e^{-\mu k}/a.
	\end{flalign*} We now compute $\dE\left[X_\bT (X_\bT-1)\right] $ by classically counting the number of ordered pairs of distinct isolated tree components of $\cG \land \cG'$ isomorphic to $\bT$. This number is then multiplied by the probability of observing these two distinct isolated components. This gives
	\begin{flalign*}
	\dE\left[X_\bT (X_\bT-1)\right] &= \binom{n}{k} \binom{n-k}{k} \times \left(\frac{k!}{a}\right)^2 \times \left(\frac{\mu}{n}\right)^{2(k-1)} \left(1-\frac{\mu}{n}\right)^{2\left(k(n-2k) + \binom{k}{2}-k+1\right)} \left(1-\frac{\mu}{n}\right)^{k^2}.
	\end{flalign*}
	Here again, $k=o(\sqrt{n})$ gives that 
	\begin{flalign*}
	\dE\left[X_\bT (X_\bT-1)\right] &\sim n^2 \mu^{2(k-1)} e^{-2\mu k}/a^2.
	\end{flalign*}Denoting $\alpha = \alpha(\bT) := n \mu^{k-1}e^{-\mu k}/a(\bT)$, these computations give that $\dE\left[X_\bT\right] \sim \Var\left(X_\bT\right) \sim \alpha(\bT)$ when $n \to \infty$, uniformly in $k \leq K(n)$ as soon as $K(n) = o(\sqrt{n})$. Let us fix $\eps = \eps(n)>0$ small enough. Applying Chebyshev's inequality together with the union bound gives
	\begin{flalign*}
	\dP\left(\exists (k,\bT) \in [K(n)] \times \dT, X_{\bT} \leq (1-\eps)\alpha(\bT)\right) & \leq \sum_{k=1}^{K(n)} \sum_{\bT \in \dT_k} \dP\left(X_\bT - \dE\left[X_\bT\right] \leq (1- \eps) \alpha(\bT) - \dE\left[X_\bT\right] \right)\\
	& \overset{(a)}{\leq} \sum_{k=1}^{K(n)} \sum_{\bT \in \dT_k} \frac{\Var\left(X_\bT\right)}{\left((1- \eps) \alpha(\bT) - \dE\left[X_\bT\right]\right)^2} \\
	& \overset{(b)}{\leq} (1+o(1)) \sum_{k=1}^{K(n)} \sum_{\bT \in \dT_k} \frac{1}{\eps^2 \alpha(\bT)} \\
	&  \overset{(c)}{\leq} (1+o(1)) \sum_{k=1}^{K(n)} \sum_{\bT \in \dT_k} \frac{1}{\eps^2 n f(k)} \\
	& \overset{(d)}{\leq}(1+o(1)) K(n)^{K(n)} \frac{1}{\eps^2 n f(K(n))}, && \\
	\end{flalign*} where 
	\begin{equation}\label{eq:def_fk}
	f(k) := \frac{\mu ^{k-1}e^{-\mu k}}{k!}.
	\end{equation}
	We used in $(a)$ that all $ (1- \eps) \alpha(\bT) - \dE\left[X_\bT\right]$ are negative for $n$ large enough, in $(b)$ uniformity in $k \leq K(n)$, in $(c)$ the lower bound $nf(k)$ for $\alpha(T)$, and finally in $(d)$ that $k \mapsto f(k)$ is decreasing since $\mu e^{-\mu}<1$.
	
	Taking now e.g. $\eps = n^{-1/4}$, the last fact to check to establish the Lemma is that $K^{K}/f(K) = o(n^{1/2})$ when $K=K(n) = \log^{1/2}(n)$:
	\begin{flalign*}
	K^{K}/f(K) &= K^K K! (1/\mu)^{K-1} e^{\mu K}\\
	& \leq \exp\left(2 K \log K + (\log(1/\mu) + \mu) K\right)\\
	&=  \exp\left(\log^{1/2}(n) \log \log n + (\log(1/\mu) + \mu) \log^{1/2}(n)\right) = o(n^{1/2}).
	\end{flalign*} 

\end{proof}

\subsection{Proof of Lemma \ref{lemma:controle_overlap_ij}}\label{app:proof_overlap}
\begin{proof}
Denote $T_\infty := \card{V_\infty}$ and $T_> := \card{V_>}$. First notice that for any permutations $\sigma_i, \sigma_j$ with $i \neq j$ generated with Algorithm \ref{algo:rec_construction}, we have the following equality:
\begin{equation}\label{eq:decompo_overlap}
    \Fix(\sigma_i, \sigma_j) = T_\infty + T_> + \sum_{k=1}^{K(n)} \sum_{\bT \in \dT_k} k \cdot \Fix(\Sigma^{(i)}_\bT, \Sigma^{(j)}_\bT),
\end{equation} where $\Sigma^{(i)}_\bT$ (resp. $\Sigma^{(j)}_\bT$) is the tree permutation associated with $\bT$ in $\sigma_i$ (resp. in $\sigma_j$). We know that $T_\infty = c(\lambda s) n +o(n)$ w.h.p. and by Lemma \ref{lemma:small_trees}, $T_> =o(n)$ w.h.p.

Define
\begin{equation}
    \Fix'(\sigma_i, \sigma_j) := \sum_{k=1}^{K(n)} \sum_{\bT \in \dT_k} k \cdot \Fix(\Sigma^{(i)}_\bT, \Sigma^{(j)}_\bT),
\end{equation} the second term in \eqref{eq:decompo_overlap}. We dominate $\Fix'(\sigma_i, \sigma_j)$ as follows: \begin{lemme}\label{lemma:laplace_overlap}
	If $X=\Fix(\Sigma^{(i)}_\bT, \Sigma^{(j)}_\bT),$, then for all $t \in \dR$,
	\begin{equation}\label{eq:lemma:laplace_overlap}
	\dE\left[e^{tX}\right] \leq \exp(e^t).
	\end{equation}
\end{lemme}
\begin{proof}
	\begin{flalign*}
	\dE\left[e^{tX}\right] & = \sum_{m \geq 0} e^{tm} \dP(X \geq m).
	\end{flalign*} Noting that $\dP(X \geq m) \leq \dE\left[\binom{X}{m}\right]$ and that
	\begin{flalign*}
	\dE\left[\binom{X}{m}\right] &= \frac{1}{m!} \dE\left[X(X-1)\ldots (X-m+1)\right]\\
	& = \frac{1}{m!} k(k-1)\ldots (k-m+1) \frac{(k-m)!}{k!} = \frac{1}{m!}
	\end{flalign*} gives
	\begin{flalign*}
	\dE\left[e^{tX}\right] & \leq \sum_{m \geq 0} \frac{e^{tm}}{m!}\leq \exp(e^t).
	\end{flalign*} 
\end{proof}

Using independence of the $X$ variables, Equation \eqref{eq:lemma:laplace_overlap} of Lemma \ref{lemma:laplace_overlap} give that for all $t \in \dR$,
\begin{flalign}\label{eq:laplace_overlap_total}
\dE\left[e^{t \cdot \Fix'(\sigma_{i},\sigma_{j})}\right]& \leq \prod_{k=1}^{K(n)} \prod_{\bT \in \dT_k} \exp(e^{tk})  \leq \exp\left(e^{t K(n) } K(n)^{K(n)+1}\right).
\end{flalign} Now, we use the classical Chernoff bound, for positive $t$,
\begin{flalign*}
\dP\left(\Fix'(\sigma_i,\sigma_j) \geq n^{\alpha}\right) & \leq \exp\left(- tn^{\alpha} + e^{t K(n)} K(n)^{K(n)+1}\right) \\ 
& \leq \exp\left(- \frac{n^{\alpha}}{K(n)} \left[\log \left(\frac{n^{1-\alpha}}{K(n)^{K(n)+2}}\right) - 1\right]\right),
\end{flalign*} taking $t = \frac{1}{K(n)} \log \left(\frac{n^{\alpha}}{K(n)^{K(n)+2}}\right)$.
The right hand side tend to $0$ for any $\alpha \in (0,1)$, and a simple use of the union bound ends the proof. 
\end{proof}

\subsection{Proof of Lemma \ref{lemma:controle_F}}\label{app:proof_control_F}
\begin{proof}
Fix $t>0$. We use a standard first moment method. We will use the results of Lemmas \ref{lemma:small_trees} and \ref{lemme:controle_X}, conditioning on the event $\cA$ where the corresponding results hold. Since $\dP(\cA) = 1-o(1)$, this conditioning is legitimate for our purpose. 

\paragraph{Step 1.} Let us first control the term $F(\cS_{out}, \sigma_{i_1},\ldots,\sigma_{i_r})$: edges of $\cS_{out}$ are made of exactly one vertex in $V_{\infty,>}$. There are at most $n^2$ such edges, and the probability for a given edge of $\cS_{out}$ being a common fixed edge of $\sigma_{i_1},\ldots,\sigma_{i_r}$ is $\frac{1}{X_\bT^{r-1}}$, which can be upper-bounded on $\cA$ by $(nf(K(n)))^{1-r} \leq n^{1-r+t/2}$ by Remark \ref{remarque:f(K)}. 

Edges of $\cS_{out}$ thus have a contribution in $\dE\left[F(\sigma_{i_1},\ldots,\sigma_{i_r}) | \cA \right]$ of at most $n^{3-r+t/2}$.

\paragraph{Step 2.} In the edges appearing in $F(\sigma_{i_1},\ldots,\sigma_{i_r})$, we consider three cases:
\begin{itemize}
\item[$(i)$] edges of $\intra$: these are edges made with two vertices in the same tree $\cT \sim \bT \in \dT$. On event $\cA$, there are at most $$\sum_{k=1}^{K(n)} \sum_{\bT \in \dT_k} X_\bT k^2 \leq n K(n)$$ such edges. The probability for a given edge of $\intra$ made of vertices of  $\bT \in \dT$ being a common fixed edge of $\sigma_{i_1},\ldots,\sigma_{i_r}$ is $\frac{1}{X_\bT^{r-1}}$, which can be upper-bounded by $(nf(K(n)))^{1-r} \leq n^{1-r+t/2}$. Edges of $\intra$ thus have a contribution in $\dE\left[F(\sigma_{i_1},\ldots,\sigma_{i_r}) | \cA \right]$  of at most $n^{2-r+t/2}$.

\item[$(ii)$] edges of $\interu$: these are edges made with two vertices $u,v$ in different trees $\cT \neq \cT'$ (but that may be $\sim$ to the same $\bT \in \dT$), and verifying $u \notsimt v$. There are at most $n^2$
such edges. Since $u \notsimt v$, there are only one possibility to map two edges of $\interu$. The probability for a given edge of $\interu$ made of vertices of  $\cT \sim \bT, \cT' \in \bT'$ being a common fixed edge is $\frac{1}{(X_\bT(X_\bT -1))^{r-1}}$, and edges of $\interu$ thus have a contribution in the expectation of at most $n^{4-2r+t/2}$.

\item[$(iii)$] edges of $\interd$: these are edges similar to case $(ii)$, except that their endpoints belong necessarily to isomorphic trees, and verifying $u \simt v$. There are at most $n^2$ 
such edges. Since $u \simt v$, there are two ways to map two edges of $\interd$. The probability for a given edge of $\interd$ made of vertices of  $\cT, \cT' \sim \bT$ being a common fixed edge is time $\left(\frac{2}{X_\bT(X_\bT -1)}\right)^{r-1}$, and edges of $\interd$ thus have a contribution in the expectation of at most $n^{4-2r+t/2}$.
\end{itemize}

\paragraph{Step 3.} The first two steps show that $\dE\left[F(\sigma_{i_1},\ldots,\sigma_{i_r}) | \cA \right] \leq C n^{3-r+t/2}$ for all $t>0$. Summing over all possible $r$-tuples of permutations, Markov inequality yields
\begin{flalign*}
\dP\left( \exists r \geq 4, \, \exists \sigma_{i_1}, \ldots, \sigma_{i_r} \mbox{ pairwise distinct}, \, F(\cS,\sigma_{i_1},\ldots,\sigma_{i_r}) \geq 1 \right) & \leq o(1) + \sum_{r=4}^{\infty} p^r C n^{3-r+t/2}\\
& \leq C p^4 n^{t/2-1} \to 0,
\end{flalign*} for $t$ small enough, and 
\begin{flalign*}
\dP\left( \exists \sigma_{i_1}, \sigma_{i_2}, \sigma_{i_3} \mbox{ pairwise distinct}, \, F(\cS,\sigma_{i_1}, \sigma_{i_2}, \sigma_{i_3}) \geq n^t \right) & \leq o(1) + p^3 \times C n^{-t/2} \to 0, 
\end{flalign*}and
\begin{flalign*}
\dP\left( \exists \sigma_{i_1} \neq \sigma_{i_2}, F(\cS,\sigma_{i_1}, \sigma_{i_2}) \geq n^{1+t} \right) & \leq o(1) + p^2 \times C n^{-t/2} \to 0.
\end{flalign*}
\end{proof}

\end{document}